\documentclass[sigconf,screen, nonacm]{acmart}

\usepackage{amsthm}
\usepackage{amsmath}
\usepackage{mathtools}
\usepackage{graphicx}
\usepackage{subcaption}
\usepackage{algorithm}
\usepackage{algorithmic}
\usepackage{balance}
\usepackage{tcolorbox}
\definecolor{dkred}{rgb}{0.8,0.,0.}
\definecolor{dkblue}{rgb}{0,0.08,0.45}
\newcommand{\red}[1]{\textcolor{dkred}{#1}}

\usepackage{wrapfig}
\theoremstyle{plain}
\newtheorem{condition}{Condition}[section]

\newcommand{\mE}{\mathbb{E}}

\newcommand{\calD}{\mathcal{D}}
\newcommand{\calX}{\mathcal{X}}
\newcommand{\calF}{\mathcal{F}}
\newcommand{\calA}{\mathcal{A}}

\newcommand{\nabt}{\nabla_{\theta}}

\newcommand{\trueV}{V(\pi_{\theta})}
\newcommand{\ips}{\nabt \widehat{V}_{\mathrm{IPS}}(\pi_\theta; \mathcal{D})}
\newcommand{\dr}{\nabt \widehat{V}_{\mathrm{DR}}(\pi_\theta; \mathcal{D})}

\newcommand{\lcpi}{\nabt \widehat{V}_{\mathrm{LCPI}}(\pi_{\theta}; \mathcal{D})}

\AtBeginDocument{%
  \providecommand\BibTeX{{%
    \normalfont B\kern-0.5em{\scshape i\kern-0.25em b}\kern-0.8em\TeX}}}

\begin{document}

\title{Offline Contextual Bandits in the Presence of New Actions}

\author{Ren Kishimoto} 
\affiliation{  
    \institution{Institute of Science Tokyo}
    \city{Tokyo}
    \country{Japan}
}
\email{kishimoto.r.ab@m.titech.ac.jp}

\author{Tatsuhiro Shimizu} 
\affiliation{
    \institution{Yale University}
    \city{Connecticut}
    \country{USA}
}
\email{tatsuhiro.shimizu@yale.edu}

\author{Kazuki Kawamura} 
\affiliation{
    \institution{Sony Group Corporation}
    \city{Tokyo}
    \country{Japan}
}
\email{Kazuki.Kawamura@sony.com}

\author{Takanori Muroi} 
\affiliation{
    \institution{Sony Group Corporation}
    \city{Tokyo}
    \country{Japan}
}
\email{Takanori.Muroi@sony.com}

\author{Yusuke Narita} 
\affiliation{
    \institution{Yale University}
    \city{Connecticut}
    \country{USA}
}
\email{yusuke.narita@yale.edu}

\author{Yuki Sasamoto} 
\affiliation{
    \institution{Sony Group Corporation}
    \city{Tokyo}
    \country{Japan}
}
\email{Yuki.Sasamoto@sony.com}

\author{Kei Tateno} 
\affiliation{
    \institution{Sony Group Corporation}
    \city{Tokyo}
    \country{Japan}
}
\email{Kei.Tateno@sony.com}

\author{Takuma Udagawa} 
\affiliation{
    \institution{Sony Group Corporation}
    \city{Tokyo}
    \country{Japan}
}
\email{Takuma.Udagawa@sony.com}

\author{Yuta Saito}
\affiliation{
    \institution{Hanjuku-kaso, Co., Ltd. }
    \city{Tokyo}
    \country{Japan}
}
\email{saito@hanjuku-kaso.com}

\renewcommand{\shortauthors}{Ren Kishimoto, et al.}

\begin{abstract}
Automated decision-making algorithms drive applications in domains such as recommendation systems and search engines. These algorithms often rely on off-policy contextual bandits or \textit{off-policy learning} (OPL). Conventionally, OPL selects actions that maximize the expected reward from an existing action set. However, in many real-world scenarios, actions, such as news articles or video content, change continuously, and the action space evolves over time compared to when the logged data was collected. We define actions introduced after deploying the logging policy as \textit{new actions} and focus on the problem of OPL with new actions. Existing OPL methods identify optimal actions from the existing set effectively. However, these methods cannot learn and select new actions because no relevant data are logged. To address this limitation, we propose a new OPL method that leverages action features. In particular, we first introduce the Local Combination PseudoInverse (LCPI) estimator for the policy gradient, generalizing the PseudoInverse estimator initially proposed for off-policy evaluation of slate bandits. LCPI controls the trade-off between reward-modeling condition and the condition for data collection regarding the action features, capturing the interaction effects among different dimensions of action features. Furthermore, we propose a generalized algorithm called \textbf{\textit{Policy Optimization for Effective New Actions (PONA)}}, which integrates LCPI, a component specialized for new action selection, with Doubly Robust (DR), which excels at learning within existing actions. We define PONA as a weighted sum of the LCPI and DR estimators, optimizing both the selection of existing and new actions, and allowing the proportion of new action selections to be adjusted by controlling the weight parameter. Through extensive experiments, we demonstrate that PONA efficiently selects new actions while maintaining the overall policy performance as opposed to most existing methods that cannot select new actions.
\end{abstract}

\maketitle

\section{Introduction}
We increasingly rely on data-driven algorithms to optimize decision-makings in various domains, including recommendation systems~\cite{gilotte2018offline, saito2021counterfactual, felicioni2022off, saito2021evaluating, saito2021open}, search engines~\cite{li2015counterfactual}, advertising~\cite{bottou2013counterfactual}, and medical treatments~\cite{kallus2018policy}. These problems take the form of contextual bandits, where a policy selects actions based on observed contexts, and the corresponding rewards become available. The goal in such problems is to learn a policy that maximizes the expected reward. \textit{Off-Policy Learning} (OPL)~\cite{swaminathan2015batch,metelli2021subgaussian,swaminathan2017off,saito2024long} enables learning such policies by leveraging previously collected data, eliminating the need for online deployment and avoiding the risk of negatively impacting user experiences through active exploration~\cite{saito2021counterfactual,kiyohara2023off}. Typically, OPL focuses on identifying the best actions within a predefined set of existing actions. However, in many real-world scenarios, the action space evolves over time as new actions emerge or existing ones are updated. For example, search systems must process new articles daily, and recommendation systems for products or videos continually welcome new items. We define these actions, introduced after the data collection phase, as \textbf{\textit{new actions}}. In this work, we rigorously address the challenges of OPL regarding the presence of such new actions for the first time in the relevant literature.

Existing OPL methods can be broadly categorized into policy-based and regression-based approaches~\cite{saito2021counterfactual,saito2024potec}. Regression-based methods first predict the expected rewards of existing actions based on supervised learning on logged data, and then pick the action that has the highest estimated reward for each context~\cite{jeunen2021pessimistic}. In contrast, policy-based methods update the policy parameter based on the policy gradient~\cite{swaminathan2015counterfactual,saito2024potec}, which is estimated using techniques such as inverse propensity score (IPS)~\cite{strehl2010learning} and doubly robust (DR)~\cite{dudik2014doubly}. These methods are effective at identifying optimal actions within the set of existing ones. However, they are not designed to evaluate the effectiveness of new actions. In particular, the naive applications of existing methods cannot choose new actions at all, failing to give fair opportunities to them.

To address these critical limitations of the typical OPL methods, we consider and formulate a scenario where actions are represented as multi-dimensional action features. There exist many real-life problems where the actions are represented with features~\cite{felicioni2022off,saito2022off}. For example, in the problem of thumbnail optimization (Figure~\ref{fig:new_action}) like done at Netflix~\cite{amat2018artwork}, actions correspond to thumbnails described by features such as character type (e.g., male, female, child), title position (e.g., top, middle, bottom), and title size (e.g., large, small). In particular, the left side of the figure depicts existing thumbnails in the predefined action set, while the right side illustrates new thumbnails that are not at all observed in the logged data. As in this illustration, we consider leveraging the action features to develop a new OPL method that effectively selects new actions (i.e., unobserved combinations of action features) while still achieving competitive expected rewards as the overall policy performance.

Merely applying existing regression-based or policy-based approaches to observable action features does not resolve the problem of new actions. This is because the combinations of action features that represent new actions are never observed in the logged data.\footnote{We will indeed show empirically that existing OPL methods cannot handle new actions that have no observable rewards even with the use of action features.} A possible solution might be to draw inspiration from the PseudoInverse (PI) method introduced in the slate bandit setting~\cite{swaminathan2017off,vlassis2021control,kiyohara2024off}. Slate bandits refer to the setting where a policy aims to optimize the selection of slates, which consist of multiple action features (which are also called sub-actions)~\cite{swaminathan2017off}. To deal with the combinatorial slate spaces and resulting variance issues, the PI estimator relies on two conditions~\cite{swaminathan2017off,vlassis2021control}: (1) \textit{the logging policy provides full support for each action feature independently}, and (2) \textit{the expected reward function (q-function) is a linear combination of the intrinsic value of each feature dimension}. Under these conditions, the PI estimator can estimate the policy value or policy gradient without bias even if there are new actions. However, the linearity condition is particularly restrictive because it treats the effects of action features as completely independent, ignoring potentially significant interaction effects and thus introducing bias into the estimation~\cite{kiyohara2024off}. Therefore, we relax the linearity condition of PI to deal with local interactions among action features and propose \textbf{Local Combinatorial Feature Interaction (LCPI)}—a generalized version of the PI estimator. LCPI enables new policies to account for interactions among different dimensions of action features, allowing a more effective selection of new actions. Although LCPI is better at identifying effective new actions, existing OPL methods, such as the policy-based approach using DR to estimate the policy gradient~\cite{dudik2014doubly}, can be more effective for learning within existing actions. This is because they do not rely on any conditions about the form of the q-function. To leverage the strengths of both approaches, we finally introduce a hybrid algorithm named \textbf{\textit{Policy Optimization for Effective New Actions (PONA)}}, which integrates LCPI and DR by weighting them through a weight parameter. This design allows PONA to optimize the overall policy performance via the effective selection of existing actions while identifying promising new actions. Furthermore, by adjusting the weight parameter, PONA can control how aggressively it selects new actions, supporting both conservative (greater focus on existing actions) and aggressive (greater focus on new actions) strategies. Results on synthetic and real-world experiments demonstrate that PONA efficiently selects new actions and achieves satisfactory overall performance, whereas existing methods fail to choose new actions severely.

The key contributions of our work can be summarized as follows.
\begin{itemize}
    \item We propose a new policy gradient estimator, LCPI. It relaxes the linearity condition of PI in slate bandits by accounting for the interaction effects among action features, enabling a more effective selection of new actions.
    \item We introduce a hybrid method called PONA that integrates LCPI with an existing method (DR) through a weight parameter. This design allows us to balance the overall policy performance and selection of new actions.
    \item We conduct comprehensive experiments and show that PONA can select new actions effectively while maintaining satisfactory overall performance in OPL.
\end{itemize}

\begin{figure}[t]
    \centering
    \includegraphics[width=0.8\linewidth]{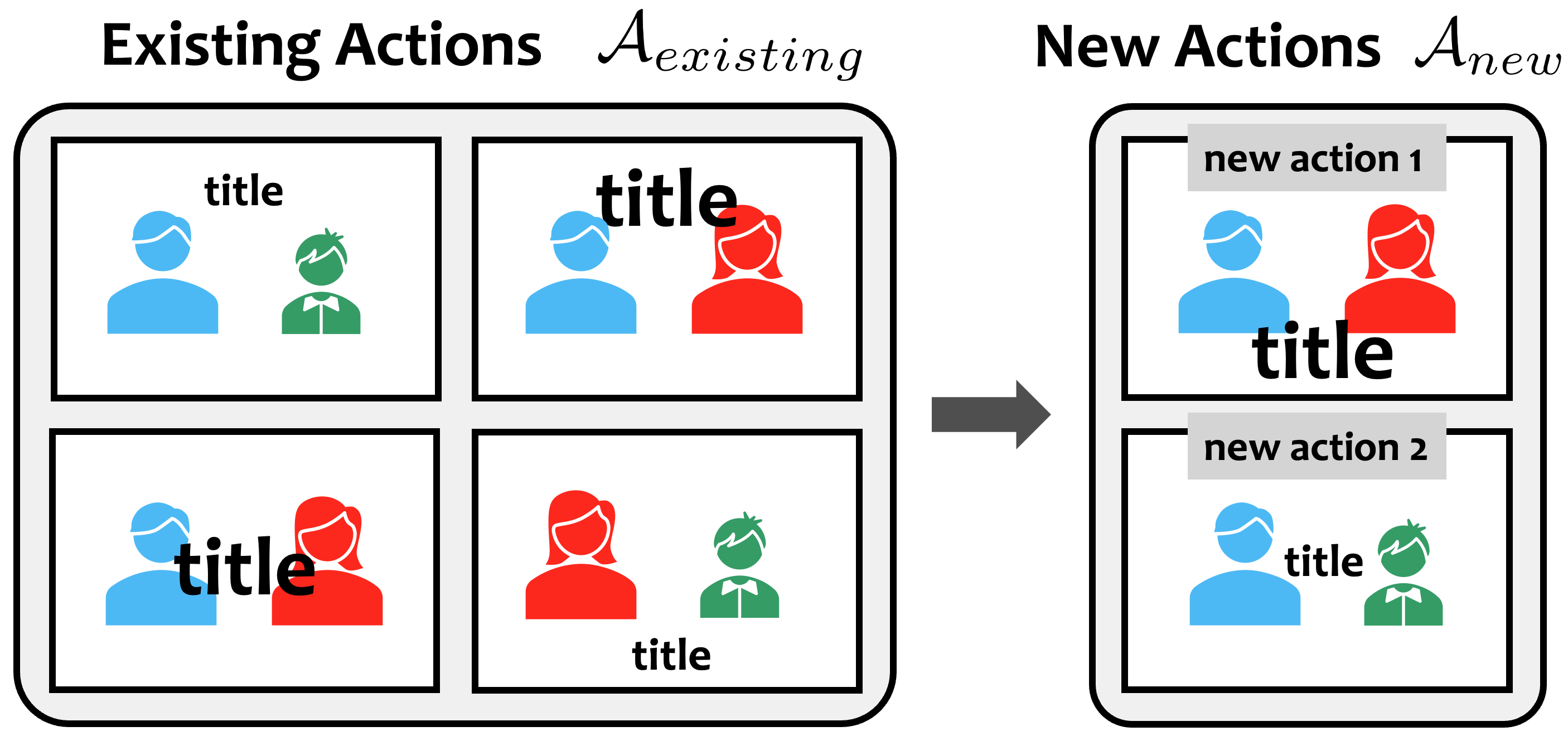} \vspace{-2mm}
    \caption{Examples of Existing and New Actions.}
    \raggedright
    \fontsize{9pt}{9pt}\selectfont \textit{Note}: 
    There are three types of action features: \textbf{character type} (chosen as two from male, female, or child), \textbf{title position} (top, center, or bottom), and \textbf{title size} (large or small). The left side of the figure illustrates the thumbnails for existing actions, while the right side shows examples of new actions: "Male and female, title position at the bottom, large title" (\textbf{new action 1}) and "Male and child, title position at the center, small title" (\textbf{new action 2}).
     \label{fig:new_action} \vspace{-4mm}
\end{figure}

\section{Preliminaries: Conventional OPL}
This section formulates the typical problem of OPL without new actions and introduces existing OPL methods.

First, let $x \in \mathcal{X} \subseteq \mathbb{R}^{d_x}$ represent a context vector such as user features in recommender systems. Given a context $x$, a possibly stochastic policy $\pi(a | x)$ selects an action $a \in \calA$ such as products or movies. Furthermore, let $r \in \mathbb{R}_{\ge0}$ denote a reward, such as conversion label and play duration, drawn from an unknown conditional distribution $p(r | x, a)$. In the setting of OPL, we are provided with a logged dataset $\mathcal{D} := \{(x_i, a_i, r_i)\}_{i=1}^n$ collected under a logging policy $\pi_0$ such that $(x, a, r) \sim p(x)\pi_0(a|x)p(r|x, a)$.

We define the expected reward under the deployment of a policy $\pi$ (policy value) as a measure of its overall performance: 
\begin{align*} 
    V(\pi) := \mathbb{E}_{p(x)\pi(a|x)p(r|x, a)}[r] = \mathbb{E}_{p(x)\pi(a|x)}[q(x, a)], 
\end{align*} 
where $q(x, a) := \mathbb{E}[r|x, a]$ is the expected reward function or \textit{q-function} for a given context $x$ and action $a$.

The goal in OPL is to learn a new policy $\pi_\theta$, parameterized by $\theta$, using only the logged dataset $\mathcal{D}$. When learning a policy, we aim to maximize the policy value as $\theta^* = \mathrm{argmax}_{\theta \in \Theta} V(\pi_\theta).$
The following describes two typical approaches to OPL.

Firstly, the \textbf{policy-based} approach updates the policy parameter $\theta$ using iterative gradient ascent, $\theta_{t+1} \leftarrow \theta_t + \eta \nabla_\theta \trueV$, where 
\begin{align} 
    \nabla_\theta \trueV := \mathbb{E}_{p(x)}\left[\sum_{a \in \mathcal{A}} \pi_\theta(a|x) q(x, a) \nabla_\theta \log \pi_\theta(a|x)\right],
\end{align} 
is the policy gradient (PG). Since the true PG is unknown, we need to estimate it using the logged data $\calD$. A common approach for this estimation is to use the IPS method~\cite{precup2000eligibility}: 
\begin{align} 
    \ips := \frac{1}{n} \sum_{i=1}^n \frac{\pi_\theta(a_i|x_i)}{\pi_0(a_i|x_i)} r_i \nabla_\theta \log \pi_\theta(a_i|x_i). \label{eq:ips}
\end{align} 
where $\pi_\theta(a|x)/\pi_0(a|x)$ is called the \textit{importance weight}. It plays a key role in ensuring the unbiasedness of IPS but often causes the issue of high variance and inefficient policy learning~\cite{saito2024potec}.

An improved approach to estimate the PG is using the DR estimator~\cite{dudik2014doubly}. It leverages a q-function estimator $\hat{q}(x,a)$ to reduce the variance in the PG estimation compared to IPS as
\begin{align} 
    \dr := &\frac{1}{n} \sum_{i=1}^n \frac{\pi_\theta(a_i|x_i)}{\pi_0(a_i|x_i)} \left(r_i - \hat{q}(x_i, a_i)\right) \nabla_\theta \log \pi_\theta(a_i|x_i) \notag \\ 
    & \qquad + \frac{1}{n} \sum_{i=1}^n \sum_{a \in \mathcal{A}} \pi_\theta(a|x_i) \hat{q}(x_i, a) \nabla_\theta \log \pi_\theta(a|x_i). \label{eq:dr} 
\end{align} 
IPS and DR are both unbiased against the true PG under the \textit{full support} condition (\textbf{a condition to ensure sufficient data collection by the logging policy $\pi_0$}).

\begin{condition} (Full Support)  \label{condition:full_support}
    The logging policy $\pi_0$ is said to have full support if $\pi_0(a|x) > 0, \forall x \in \mathcal{X}, \forall a \in \mathcal{A}.$ 
\end{condition}

Secondly, the \textbf{regression-based} approach estimates the q-function $q(x, a)$ using off-the-shelf supervised machine learning methods. It then transforms the estimated q-function $\hat{q}_\theta(x, a)$ into a decision-making policy, for example, by applying the softmax function as
\begin{align}
    \pi_\theta (a|x) = \frac{\exp(\hat{q}_\theta(x, a))}{\sum_{a'\in\calA} \exp(\hat{q}_\theta(x, a'))} \label{eq:reg_a}
\end{align}
This approach may fail due to bias issues resulting from the difficulty in accurately estimating the q-function $q(x,a)$. However, it can avoid the issue of high variance compared to the policy-based approach based on IPS and DR~\cite{jeunen2021pessimistic}.

Existing policy-based or regression-based methods perform effectively in identifying existing actions with high expected rewards~\cite{swaminathan2015batch,metelli2021subgaussian,saito2024potec}. However, these methods cannot learn to select new actions, even if some of the new actions have high expected rewards. This is because we do not observe any data about new actions in $\calD$, which severely violates \textit{full support} (Condition~\ref{condition:full_support}). Therefore, we need to first relax the typical condition of full support to effectively learn policies that can choose promising new actions.

\section{Related Work} \label{sec:related}
This section summarizes the key related work and their connections to our problem of OPL with new actions. A more thorough discussion on related work can be found in Appendix~\ref{app:related}.

Off-Policy Evaluation (OPE) and Learning (OPL)~\cite{dudik2014doubly, wang2017optimal, liu2018breaking, farajtabar2018more, su2019cab, su2020doubly, kallus2020optimal, metelli2021subgaussian, saito2022off, saito2023off} aim to evaluate and optimize decision-making policies using only historical logged data. Various methods and estimators have been developed in OPE to control the bias and variance of the estimation of the performance of new policies~\cite{su2020doubly, su2019cab, swaminathan2015counterfactual, swaminathan2015self, dudik2014doubly, farajtabar2018more, kallus2020optimal, kiyohara2022doubly, kiyohara2023off, metelli2021subgaussian, saito2021counterfactual, wang2017optimal, kiyohara2024off}. In OPL, these estimation techniques enable accurate estimation of the policy gradient~\cite{saito2024potec}. Unfortunately, most existing methods in OPL rely crucially on full support (Condition~\ref{condition:full_support}) and focus solely on selecting the most effective actions within those already present in the logged data~\cite{sachdeva2020off}. Thus, they cannot learn and give fair opportunities to new actions that are potentially available in many decision-making problems.

A setting particularly relevant to ours is OPE for slate bandits. It considers selecting a combinatorial action called \textit{slate} composed of multiple sub-actions. For example, in medical treatment scenarios, one may aim to select the optimal combination of drug doses to improve patient outcomes. While these systems often have access to extensive logged data, an accurate OPE is challenging due to the exponential increase in variance associated with slate-wise importance weighting~\citep{swaminathan2017off, su2020doubly, kiyohara2024off}. To address this challenge, several versions of \textit{PseudoInverse} (PI) estimators have been proposed~\citep{su2020doubly, swaminathan2017off, vlassis2021control}. They use slot-wise importance weights to relax the support condition and reduce the variance of typical estimators such as IPS. They have also been proven to enable unbiased OPE under the \textit{linearity} assumption on the q-function. Linearity requires that the q-function be linearly decomposable, ignoring interaction effects among different slots. While PI often outperforms IPS under linear reward structures, it introduces significant bias when the linearity assumption does not hold~\cite{kiyohara2024off,shimizu2024effective}. To effectively handle new actions in OPL by leveraging action features, we extend the idea of PI to estimate the policy gradient by relaxing the linearity assumption on the reward. This relaxation is crucial for dealing with the real-world non-linearities in reward functions during policy learning. It should be noted that the key contributions of our work are substantially different from those of existing studies for slate bandits~\cite{swaminathan2017off, kiyohara2024off, vlassis2021control, chaudhari2024distributional}. We focus on the problem of policy \textit{learning with new actions}. In contrast, the existing literature around slate OPE addresses only the problem of policy \textit{evaluation without considering new actions}.

Another setting similar to new actions is OPL with deficient actions~\cite{sachdeva2020off,felicioni2022off}, which refer to the actions that have zero probability of being selected under the logging policy for \textit{certain} contexts $x$. Rigorously, the set of deficient actions is written as $\{a \in \calA \mid \exists x \in \calX, \pi_0(a \mid x) = 0\}$. Deficient action is a broader concept that encompasses the new actions we consider as a (more challenging) special case. Due to this distinction, existing OPL approaches to deal with deficient actions~\cite{sachdeva2020off,felicioni2022off} cannot handle new actions.

\begin{figure}[t]
     \centering
     \includegraphics[width=0.75\linewidth]{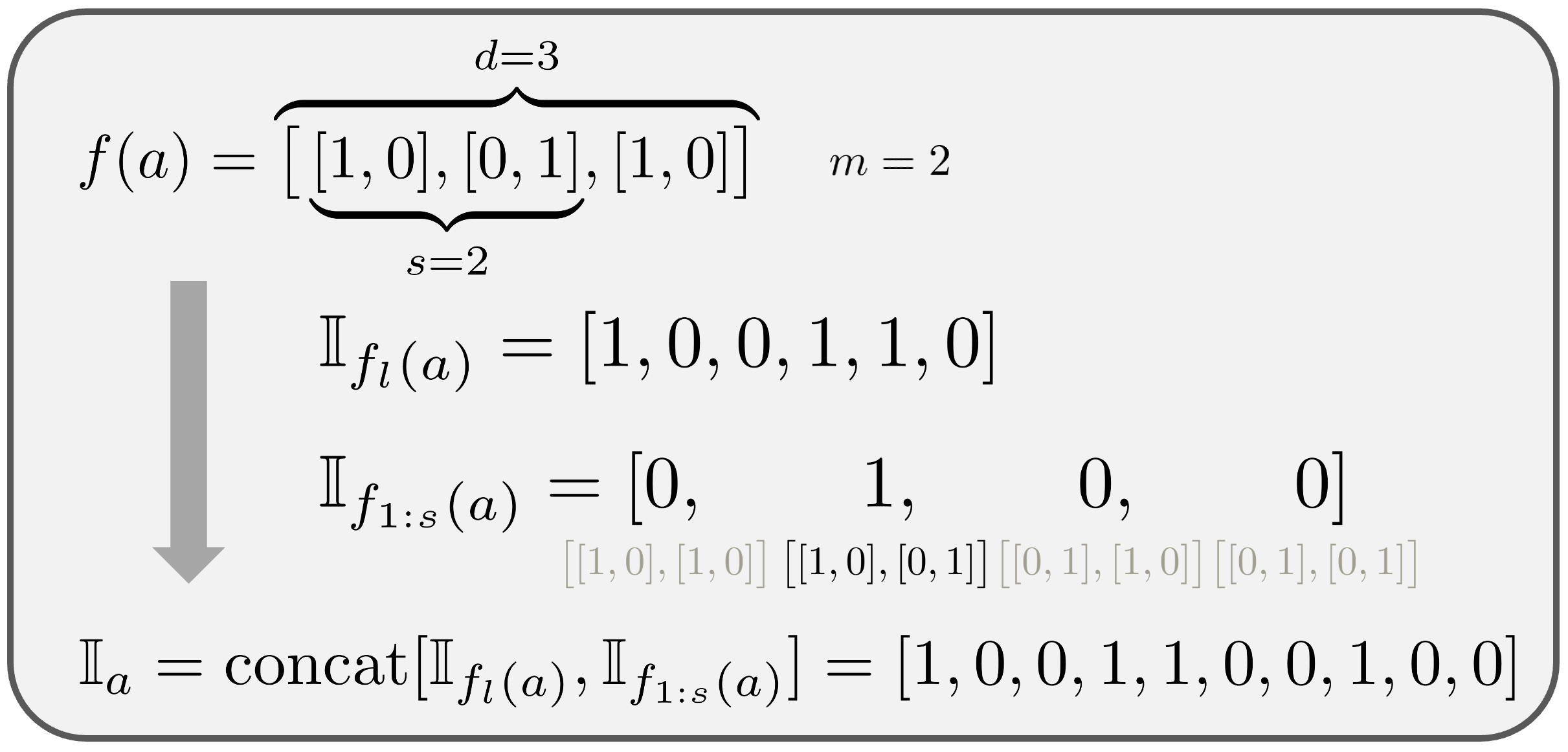} \vspace{-2mm}
    \caption{An example of converting action features into a vector with binary values by representing each dimension ($f_l(a)$) by respective one-hot vectors. The vector representing the independent effect of each dimension ($\mathbb{I}_{f_l(a)}$) is then concatenated with $\mathbb{I}_{f_{1:s}(a)}$, which represents interaction effects.} \label{fig:action_indicator} \vspace{-5mm}
\end{figure}

\section{The Proposed Approach}
The previous sections discussed the challenges associated with selecting new actions using existing OPL methods. To overcome, we propose a solution that leverages action features, which are often available in real-world decision-making problems~\citep{saito2022off,peng2023offline,felicioni2022off}.

\subsection{Formulation of OPL with Action Features}
We now represent action $a$ by the $d$-dimensional action features 
$$f(a) = (f_1(a), \dots, f_l(a), \dots, f_d(a)),$$ 
where $f_l(a) \in \mathcal{F}_l$ denotes the $l$-th dimension. In particular, we consider the setting where each action feature $f_l(a)$ is discrete and can be represented as a one-hot vector. For example, the genre feature such as `documentary', `romance', and `horror' are encoded as $[1, 0, 0]$, $[0, 1, 0]$, and $[0, 0, 1]$, respectively, as in Figure~\ref{fig:action_indicator}.

In OPL with action features, we can possibly address the presence of new actions by being able to represent them as combinations of already observed factors. Figure~\ref{fig:new_action} illustrates an example in the context of thumbnail selection with the 3-dimensional action features. The left side of the figure shows the existing action set, which includes thumbnails already considered by the logging policy. The right side introduces new thumbnails as new actions. For example, \textbf{new action 1} consists of "male and female characters," "title positioned at the bottom," and "large title size." We can see that we have already observed every individual feature of \textbf{new action 1} separately in the existing action set. \textbf{Despite this, existing OPL methods struggle to handle these new actions because the logged data does not include specific combinations of action features representing the new actions}. Therefore, we propose a novel method and algorithm that leverage these action features through a new estimator for the PG to identify effective new actions.

Before introducing our proposed approach, we rigorously introduce three key types of action spaces.
\begin{enumerate}
    \item \textbf{The space of \textit{all} actions}: The set of all possible combinations of action features:
    \[
    \calA := \{a \mid a=(f_1, \dots, f_d), \; \forall f_1 \in \calF_1, \dots, \forall f_d \in \calF_d\}.
    \]
    \item \textbf{The space of \textit{existing} actions}: The set of actions that have some positive probability under the logging policy:
    \[
    \calA_{existing} := \{a \in \calA \mid \exists x \in \calX, \, \pi_0(a | x) > 0\}.
    \]
    \item \textbf{The space of \textit{new} actions}: The set of actions that are not at all chosen under the logging policy:
    \[
    \calA_{new} := \{a \in \calA \mid \forall x \in \calX, \, \pi_0(a | x) = 0\}.
    \]
\end{enumerate}

These action spaces satisfy the following relations:
\[
    \calA_{existing} \cap \calA_{new} = \emptyset, \quad \calA_{existing} \cup \calA_{new} = \calA.
\]

In the typical formulation for OPL, the objective is to identify actions that maximize the expected reward from a set of \textit{existing} actions $\calA_{existing}$. In contrast, we consider a more general and realistic scenario, where the goal is to select actions that maximize the expected reward from $\calA$, which includes new actions $\calA_{new}$.

\subsection{The Proposed Algorithm: PONA}
\label{sec:pona}
To design a new algorithm for OPL with new actions, we first draw inspiration from the PI estimator initially proposed for OPE in slate bandits~\cite{swaminathan2017off} as described in Section~\ref{sec:related}, and provide a generalization. 

To introduce PI, we begin with formally defining an one-hot vector for each action feature as $\mathbb{I}_{f_l(a)} \in \mathbb{R}^{dm}$.\footnote{Here, we assume that each dimension of the action feature has the same number of possible values ($m=|\calF_l|$) for ease of exposition. However, this is not strictly necessary, and different dimensions can indeed have varying sizes in practice.} As shown in Figure~\ref{fig:action_indicator}, this is a flattened vector that represents the one-hot encoding of each action feature. Using this notation, we can first extend the PI estimator~\cite{swaminathan2017off} as an estimator for the PG as follows.
\begin{align}
    &\nabla_\theta \widehat{V}_{\mathrm{PI}}(\pi_\theta; \mathcal{D}) \notag \\ 
    &= \frac{1}{n}\sum_{i=1}^n \left(\sum_{a \in \calA}\pi_\theta(a|x_i) \nabla_\theta \log\pi_\theta(a|x_i) \red{\mathbb{I}_{f_l(a)}}^{T} \right)\Gamma_{\pi_0, x_i}^\dagger \red{\mathbb{I}_{f_l(a_i)}} r_i, \label{eq:PI}
\end{align}
where $\Gamma_{\pi_0, x} := \mathbb{E}_{\pi_0(a|x)}[\mathbb{I}_{f_l(a)} \mathbb{I}_{f_l(a)}^T|x]$, and $M^\dagger$ denotes the Moore-Penrose pseudoinverse of the matrix $M$.

The PI estimator introduced above relies on the following conditions (\ref{condition:independent_support} and \ref{condition:linearity}) to ensure unbiased estimation of the PG:
\begin{condition} (Independent Support) \label{condition:independent_support}
    The logging policy $\pi_0$ ensures independent support if $\pi_0(f_l|x) > 0, \forall x \in \mathcal{X}, \forall l \in [1, ..., d], \forall f_l \in \mathcal{F}_l$. Note that $\pi_0(f_l|x) = \sum_{a \in \calA: f_l(a) = f_l} \pi_0(a\,|\,x)$ is the marginal probability of observing $f_l$ under $\pi_0$.
\end{condition}

This condition requires the logging policy to independently support every dimension of the action features. The condition of Independent Support is weaker than the Full Support condition (Condition~\ref{condition:full_support}), which requires the logging policy to choose every combination of action features. When new actions exist, the Full Support condition cannot hold, but the Independent Support condition may still hold even in the presence of new actions. \textbf{Due to this relaxation of the support condition, the PI estimator can learn a new policy that chooses new actions}.

The next condition of PI is with regard to the q-function.

\begin{condition} (Linearity) \label{condition:linearity} 
For each context $x \in \calX$, there exists an (unknown) intrinsic reward vector $\phi_{x,l} \in \mathbb{R}^{dm}$ such that: 
\begin{align*} 
    q(x, a) = \sum_{l=1}^d q_l(x, f_l(a)) = \mathbb{I}_{f_l(a)}^T \phi_{x,l}.
\end{align*}
\end{condition} 

The linearity condition requires that the q-function be expressed as a linear combination of latent value functions for each dimension of the feature. \textbf{This condition essentially assumes no interaction effects between different dimensions of the action feature}. While extending the PI estimator to estimate the PG is straightforward as done in Eq.~\eqref{eq:PI}, the linearity condition rarely holds and introduces significant bias in PG estimation~\cite{kiyohara2024off,vlassis2021control}.

To overcome this critical limitation of the naive extension of PI, we first generalize Independent Support to account for local interactions of action features as follows.

\begin{condition} (Local Combination Support) \label{condition:local_combination_support}
    The logging policy $\pi_0$ is said to ensure local combination support if, for the the first $s$ dimensions of the action features, $\pi_0(f_{1:s}|x) > 0, \forall x \in \mathcal{X}, \forall f_{1:s} \in \prod_{j=1}^s\mathcal{F}_j$ as well as $\pi_0(f_l|x) > 0$ for the rest $s + 1 \le l \le d$.
\end{condition}

Note that $f_{1:s} = (f_1, ..., f_s)$ represents the first $s$ dimensions of the action features, where $1 \le s \le d$. The condition of Local Combination Support requires $f_{1:s}$ to be fully supported under the logging policy.\footnote{Note that we focus on the first $s$ dimensions just for ease of exposition, as the dimensions of the action feature can be arbitrarily reordered.} Condition~\ref{condition:local_combination_support} remains weaker than the original Full Support condition (Condition~\ref{condition:full_support}), because Full Support requires every dimension of the action features to be simultaneously supported (equivalent to Local Combination Support with $s=d$). When $s=1$, Local Combination Support reduces to Condition~\ref{condition:independent_support}. Along with this extension, we generalize the linearity condition for the q-function (Condition~\ref{condition:linearity}) as described below.

\begin{condition} (Local Linearity)  \label{condition:local_linearity_condition}
    Define a vector with binary values for representing the first $s$ dimensions of the action features as $\mathbb{I}_{f_{1:s}} \in \mathbb{R}^{m^s}$. The overall action indicator $\mathbb{I}_a$ is then defined as $\mathbb{I}_a := \text{concat}[\mathbb{I}_{f_l}, \mathbb{I}_{f_{1:s}}]$. We say that the q-function meets Local Linearity, if for each context $x \in \calX$, there exists an (unknown) intrinsic reward vector $\phi_x = \text{concat}[\phi_{x,l} \in \mathbb{R}^{dm}, \phi_{x,1:s} \in \mathbb{R}^{m^s}]$ such that:
\begin{align*}
    q(x, a) =  \underbrace{\sum_{l=1}^d q_l(x, f_l(a)) + q(x, f_{1:s}(a))}_{\mathbb{I}_{f_l(a)}^T \phi_{x, l} \quad +  \quad
\mathbb{I}_{f_{1:s}(a)}^T \phi_{x, 1:s} }  = \mathbb{I}_{a}^T \phi_{x}.
\end{align*}
\end{condition}

\textbf{The Local Linearity condition allows the first $s$ dimensions of the action features to have interaction effects with each other}. Therefore, this condition is weaker than linearity of PI (Condition~\ref{condition:linearity}), because linearity does not allow any interaction effects across different action features. When $s=1$, the condition of local linearity reduces to the linearity condition (Condition~\ref{condition:linearity}).

Building on these generalized conditions, we propose a new policy gradient estimator, called the Local Combination PseudoInverse (LCPI) estimator, generalizing the PI estimator~\cite{swaminathan2017off} as below.
\begin{align}
    &\nabla_\theta \widehat{V}_{\mathrm{LCPI}}(\pi_\theta; \mathcal{D}) \notag \\ 
    &= \frac{1}{n}\sum_{i=1}^n \left(\sum_{a \in \calA}\pi_\theta(a|x_i) \nabla_\theta \log\pi_\theta(a|x_i) \red{\mathbb{I}_a}^{T} \right)\Gamma_{\pi_0, x_i}^\dagger \red{\mathbb{I}_{a_i}} r_i,
\end{align}
where $\Gamma_{\pi_0, x} := \mathbb{E}_{\pi_0(a|x)}[\mathbb{I}_a \mathbb{I}_a^T|x]$.

LCPI satisfies unbiasedness against the PG $\nabt \trueV$ under Conditions~\ref{condition:local_combination_support} and \ref{condition:local_linearity_condition}, i.e., $\mE_{\calD}[\lcpi] = \nabt \trueV$.\footnote{Please see the appendix for the derivation.} This indicates that LCPI can estimate the reward of new actions while dealing with the interaction effects among action features (the original PI estimator cannot deal with the interaction effects).

Although LCPI is effective for selecting new actions than PI by dealing with the interaction effects, traditional policy- or regression-based methods may achieve better performance within existing actions. This is because they do not rely on any condition by focusing only on existing actions. Therefore, there is an interesting tradeoff between our LCPI and traditional methods. While traditional methods focus on identifying the optimal one within existing actions, LCPI is superior in choosing new actions based on local linearity. This interesting tradeoff leads us to finally develop the \textbf{Policy Optimization for New Actions (PONA)} algorithm, which integrates LCPI with the DR estimator\footnote{Note that PONA is general and can be combined with any existing estimator for the PG other than DR. We use DR as a part of our PONA algorithm because it is often the most effective method among traditional ones~\cite{saito2024potec}.} as follows:
\begin{align}
    &\nabla_\theta \widehat{V}_\mathrm{PONA}(\pi_\theta; \kappa, \calD)\notag \\  
    &= \kappa \cdot \nabla_\theta \hat{V}_\mathrm{LCPI}(\pi_\theta; \calD) + (1 - \kappa) \cdot \nabla_\theta \hat{V}_\mathrm{DR}(\pi_\theta; \calD),
\end{align}
where $\kappa \in [0, 1]$ is a hyperparameter that controls the trade-off between focusing on new actions and maximizing the overall policy value by prioritizing existing actions. A larger value of $\kappa$ increases the weight of the LCPI component, which leads to greater reliance on the local linearity condition and a stronger focus on optimizing new actions. However, this approach sacrifices the effectiveness of identifying the optimal actions within the set of existing actions. In contrast, a smaller value of $\kappa$ increases the weight of DR, reducing reliance on local linearity and shifting the focus toward optimizing existing actions. This results in a more conservative policy that rarely selects new actions. We can tune the value of $\kappa$ in a data-driven manner using a cross-validation procedure. Multiple formulations of cross-validation can be considered for tuning $\kappa$, depending on the specific objectives. For instance, one can simply maximize the policy value $V(\pi)$ on the validation set when tuning $\kappa$. Alternatively, it is possible to impose constraints on the proportions of new actions under the learned policy, allowing control over how aggressive or conservative the new policy should be, as follows.
\begin{align}
    \max_\kappa \ \hat{V}(\pi_{\theta, \kappa}; \mathcal{D}) \quad 
    \text{s.t.} \ \rho_L \leq \mathbb{E}_{p(x)}[ \sum_{a \in \mathcal{A}_{new}} \pi_{\theta, \kappa}(a | x)] \leq \rho_U, \
\label{eq:tuning_kappa}
\end{align}
where $\rho_L$ and $\rho_U$ define the desired range for the proportion of new actions. In the next section, we will empirically demonstrate that, by performing this cross-validation procedure and changing the values of $\rho_L$ and $\rho_U$, we can flexibly control the proportion between existing and new actions under the learned policy. Note that, when performing cross-validation, we can estimate the policy value $\hat{V}(\pi_{\theta, \kappa})$ of the policy induced by $\kappa$ by applying off-the-shelf OPE estimators such as IPS and DR in the validation set~\cite{saito2024hyperparameter}.

\begin{figure*}[t]
    \includegraphics[scale=0.34]{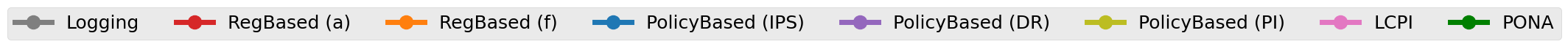}
    \includegraphics[width=0.725\linewidth]{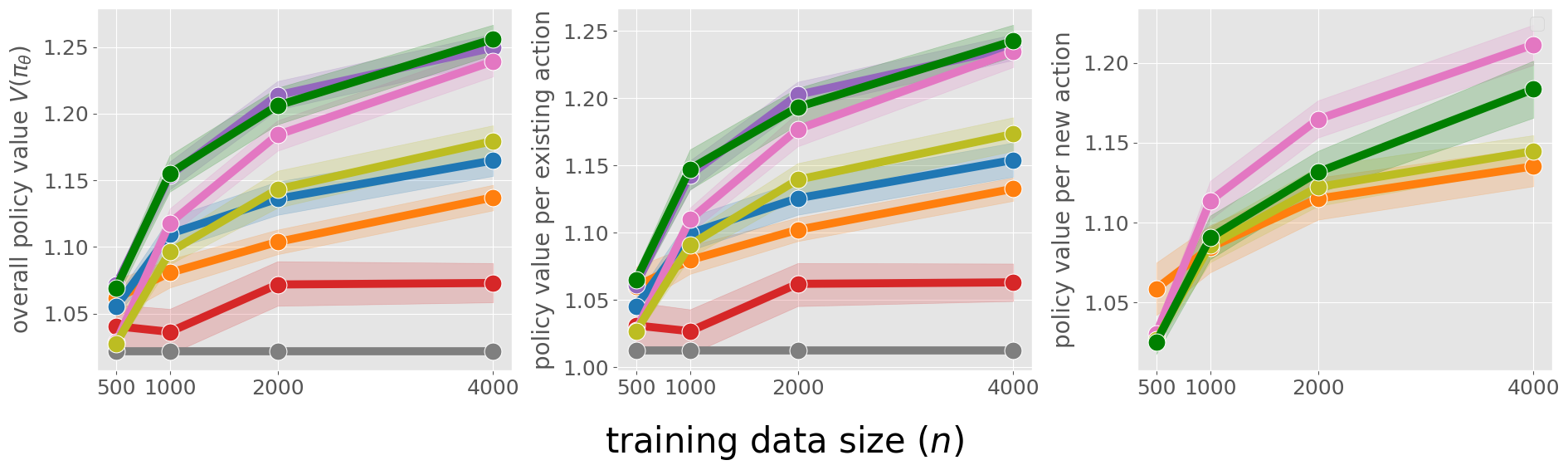} \vspace{-2mm}
    \centering
    \caption{Comparisons of the overall policy value, policy value per existing action, and policy value per new action with varying training data sizes ($n$). Note that the metrics are normalized by those of the uniform random policy.}
    \centering
    \label{fig:n_data_per}
\end{figure*}
\begin{figure*}[t]
    \includegraphics[width=0.725\linewidth]{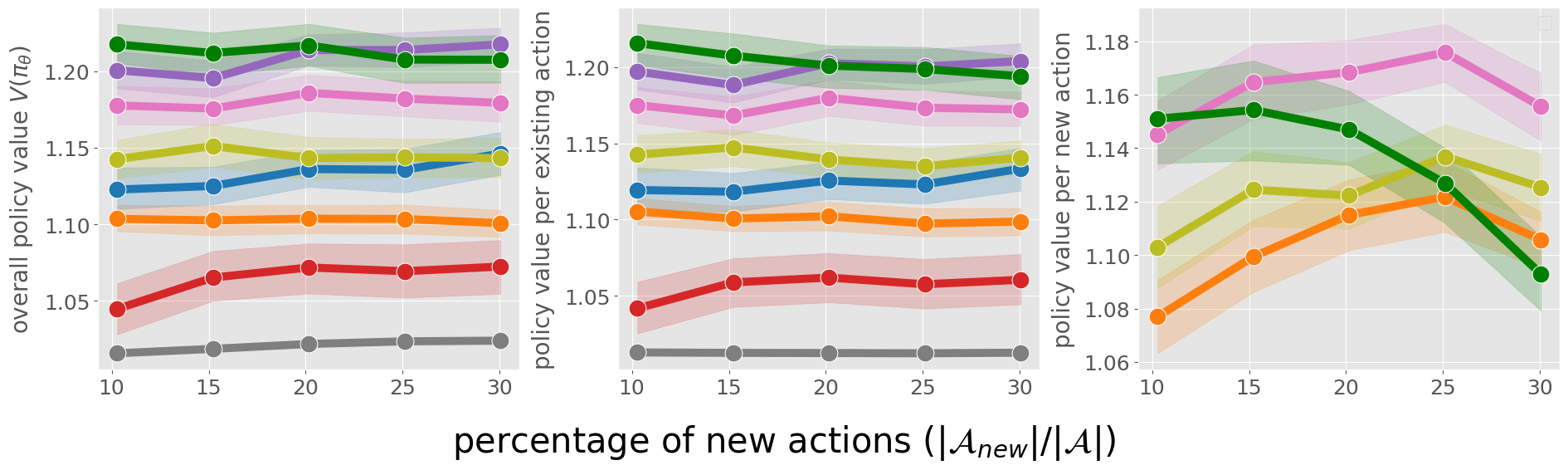} \vspace{-2mm}
    \centering
    \caption{Comparisons of the overall policy value, policy value per existing action, and policy value per new action with varying percentages of new actions ($|\calA_{new}|/|\calA|$). Note that the metrics are normalized by those of the uniform random policy.}
    \centering
    \label{fig:n_actions_per}
\end{figure*}
\begin{figure*}[t]
    \includegraphics[width=0.725\linewidth]{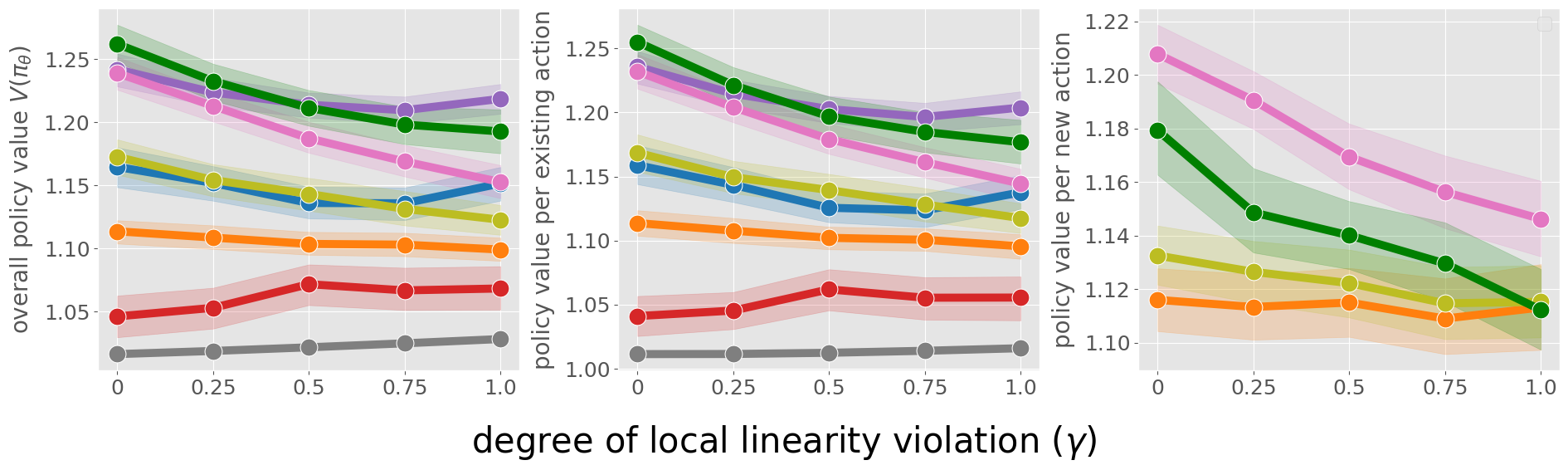} \vspace{-2mm}
    \centering
    \caption{Comparisons of the overall policy value, policy value per existing action, and policy value per new action with varying degrees of local linearity violation ($\gamma$). Note that the metrics are normalized by those of the uniform random policy.}
    \centering
    \label{fig:gamma_per}
\end{figure*}

\begin{figure*}[t]
    \includegraphics[scale=0.37]{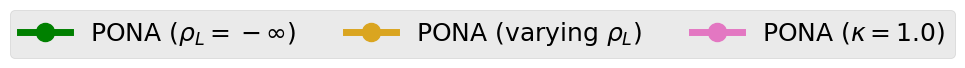}
    \includegraphics[width=0.725\linewidth]{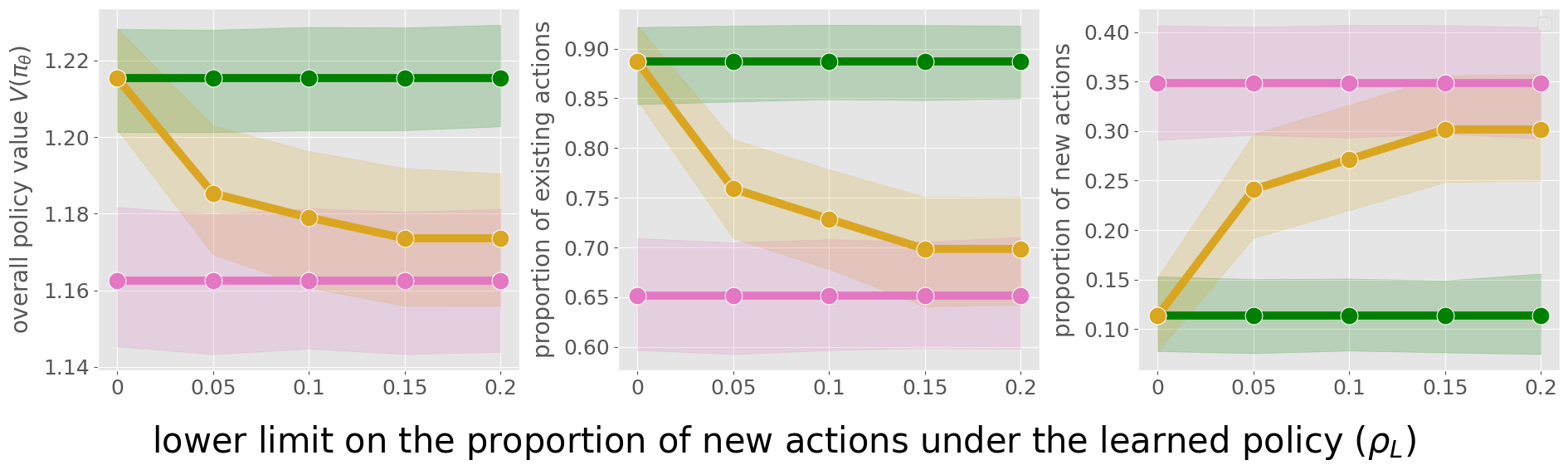}
    \centering
    \caption{Comparisons of the overall policy value, proportion of existing actions, and proportion of new actions under the learned policy with varying lower limits on the proportion of new actions ($\rho_L$ in Eq.~\eqref{eq:tuning_kappa}).}
    \centering
    \label{fig:rho}
\end{figure*}

\section{Empirical Evaluation} 
This section empirically evaluates our proposed method, PONA, and its special case (LCPI) on synthetic and real-world datasets.

\subsection{Synthetic Data} 
We first create synthetic data to evaluate policy learning algorithms and their ability to select effective new actions based on the ground-truth policy value for varying experiment configurations.

We begin by sampling context vectors $x$ from a normal distribution. The action features span $d=5$ dimensions, with each feature having $m=3$ possible values. The total number of possible actions is thus $|\mathcal{A}| = 3^5 = 243$.

We then define the synthetic q-function as follows: 
\begin{align*} 
    q(x, a) = \sum_{l=1}^d \underbrace{q_l(x, f_l(a))}_{:=x^T M_{f_l(a)}} + \underbrace{q_{1:s}(x, f_{1:s}(a))}_{:=x^T M_{f_{1:s}(a)}} + \gamma \cdot \underbrace{q_{1:d}(x, f_{1:d}(a))}_{:=x^T M_{f_{1:d}(a)}}, 
\end{align*} 
where $\gamma$ is a parameter that controls the contribution of the last term $q_{1:d}(x, f_{1:d}(a))$. $M_{f_l}$, $M_{f_{1:s}}$, and $M_{f_{1:d}}$ are parameter matrices sampled from a uniform distribution. The first term of the q-function, $q_l(x, f_l(a))$, represents the independent effect of each dimension of the action feature. The second term, $q_{1:s}(x, f_{1:s}(a))$, captures the interaction effects within the first $s$ dimensions of the action feature. Most methods, including IPS, DR, and ours, can handle the first two terms because they do not involve interaction effects beyond the first $s$ dimensions. Only the original PI method~\cite{swaminathan2017off} cannot estimate the second term due to its reliance on the linearity condition. The last term of the q-function, $q_{1:d}(x, f_{1:d}(a))$, accounts for the interaction effects among all dimensions of the action feature. The proposed methods cannot estimate the last term without bias due to their reliance on local linearity, while existing methods can achieve unbiased estimation of the q-function within existing actions regardless of the interaction effects. By adjusting the $\gamma$ parameter, we control the degree of the condition’s violation and evaluate the robustness of our methods under these violations.

Next, we define the logging policy $\pi_0$ by applying the softmax function to the expected reward function $q(x, a)$ within the set of existing actions as follows:
\begin{align*}
    \pi_0 (a|x) := \left\{
    \begin{aligned}
          &\frac{\exp(\beta\cdot q(x,a))}{\sum_{a' \in \mathcal{A}}\exp(\beta\cdot q(x,a'))} & \text{if} \quad a \in \mathcal{A}_{existing}\\
          &0 &  \text{if} \quad 
          a \in \mathcal{A}_{new}\\
    \end{aligned}
    \right.
\end{align*}
where $\beta$ is a parameter that adjusts the trade-off between optimality and entropy in the logging policy. We set $\beta = 0.05$ throughout our experiments. 
Note that the above formulation of the logging policy ensures that the new actions in $\mathcal{A}_{new}$ have zero probability of being selected by the logging policy for any context $x$.

\paragraph{\textbf{Compared Methods.}}
We compare PONA and LCPI with the logging policy $\pi_0$, the regression-based method using action index (RegBased ($a$)), the regression-based method using action feature (RegBased ($f$)), policy-based method w/ IPS (PolicyBased (IPS)), policy-based method w/ DR (PolicyBased (DR)), and policy-based method w/ PI (PolicyBased (PI)). We report RegBased ($f$) as a baseline method that estimates the q-function using action features as input. It then picks the best action according to the estimated reward function $\hat{q}_\theta(x, f_{1:d}(a))$. RegBased ($f$) differs from RegBased ($a$) in that it may select new actions. To tune the hyperparameter $\kappa$ for PONA, we perform a grid search over the range $[0, 0.25, 0.5, 0.75, 1.0]$, selecting the best value in terms of optimizing the policy value $V(\pi)$ on validation data.

\begin{figure*}[t]
    \includegraphics[scale=0.34]{image/legend_one_colum.png}
    \includegraphics[width=0.725\linewidth]{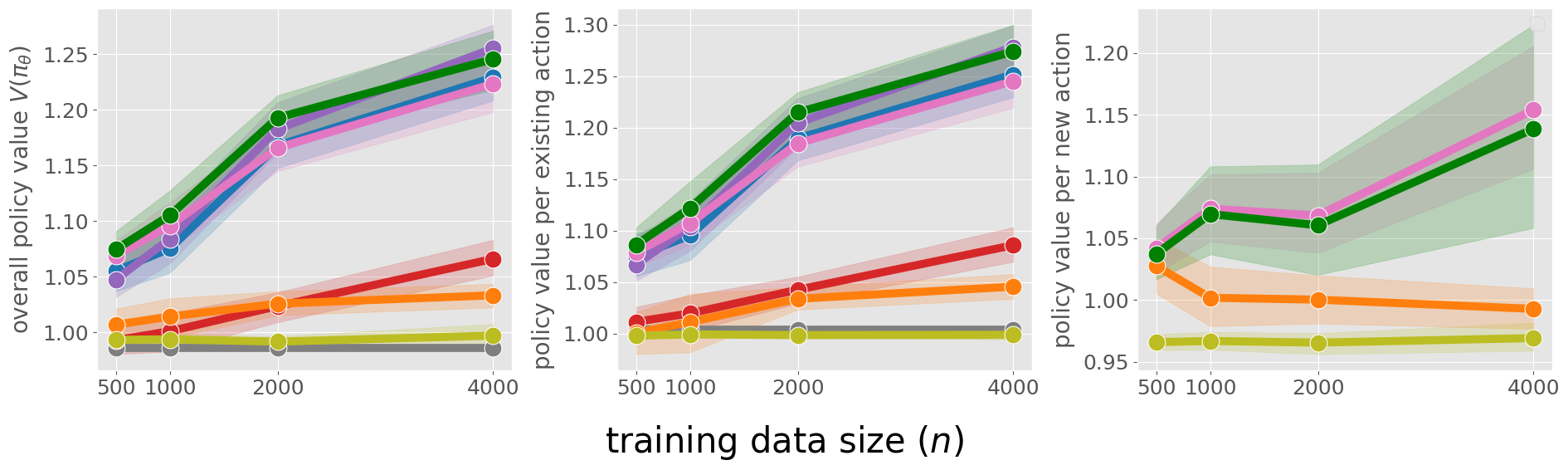} \vspace{-2mm}
    \centering
    \caption{Comparisons of the overall policy value, policy value per existing action, and policy value per new action with varying training data sizes ($n$) on the KuaiRec dataset. Note that the metrics are normalized by those of the uniform random policy.}
    \centering
    \label{fig:real_n_data}
\end{figure*}
\begin{figure*}[t]
    \includegraphics[width=0.725\linewidth]{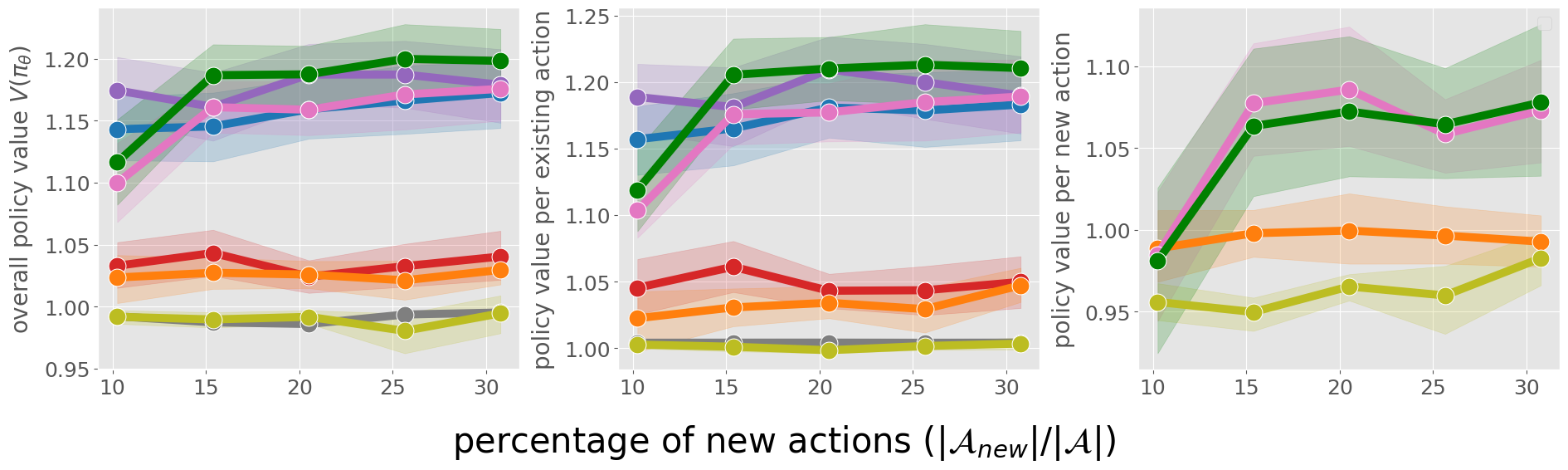} \vspace{-2mm}
    \centering
    \caption{Comparisons of the overall policy value, policy value per existing action, and policy value per new action with varying percentages of new actions ($|\calA_{new}|/|\calA|$) on the KuaiRec dataset. Note that the metrics are normalized by those of the uniform random policy.}
    \centering
    \label{fig:real_n_actions}
\end{figure*}

\paragraph{\textbf{Result.}}
We report the following metrics of policies learned by each OPL method, averaged across 200 simulations.
\begin{itemize}
    \item \textbf{overall policy value}: $ \mE[\sum_{a \in \calA} \pi_\theta(a|x) q(x, a)] $
    \item \textbf{policy value per existing action}: $ \frac{\mE[\sum_{a \in \calA_{existing}} \pi_\theta(a|x) q(x, a)]}{\mE[\sum_{a \in \calA_{existing}} \pi_\theta(a|x)]} $
    \item \textbf{policy value per new action}: $ \frac{\mE_{p(x)}[\sum_{a \in \calA_{new}} \pi_\theta(a|x) q(x, a)]}{\mE[\sum_{a \in \calA_{new}} \pi_\theta(a|x)]} $
\end{itemize}
\textbf{Overall policy value} measures the overall effectiveness of the learned policy. \textbf{Policy value per existing action} evaluates the quality of selecting existing actions, while \textbf{policy value per new action} assesses that of selection within new actions. \textbf{Note that the logging policy $\pi_0$, RegBased ($a$), PolicyBased (IPS), and PolicyBased (DR) do not choose new actions at all, and thus their results do not appear in the figures of \underbar{\textbf{policy value per new action}}}. 

\paragraph{\textbf{How does PONA perform with varying sizes of training data?}}
First, we compare the OPL methods using varying sizes of training logged data $(n \in \{500, 1000, 2000, 4000\})$, as shown in Figure \ref{fig:n_data_per}. The figure shows that most methods, including PONA and LCPI, improve in overall policy value, policy value per existing action, and policy value per new action as the training data size increases, which aligns with expectations. Notably, LCPI performs particularly well in terms of the selection of new actions when the training data size is relatively large ($n = 2000, 4000$). However, it underperforms PolicyBased (DR) in terms of overall policy value. In contrast, PONA combines the strengths of LCPI and PolicyBased (DR), maintaining a policy value per new action comparable to RegBased ($f$) while matching PolicyBased (DR) in terms of overall policy value. Considering that existing methods, including PolicyBased (DR), cannot select new actions at all, it is noteworthy that \textbf{PONA not only selects new actions and gives fair opportunities to them as effectively as RegBased ($f$) but also achieves overall performance competitive with PolicyBased (DR).}

\paragraph{\textbf{How does PONA perform with varying numbers of new actions?}}
Next, Figure \ref{fig:n_actions_per} illustrates how the OPL algorithms perform as the percentage of new actions, $100 \cdot\frac{|\calA_{new}|}{|\calA|}$, increases. As the percentage of new actions grows, the number of existing actions decreases. The figure shows that methods focusing exclusively on existing actions (RegBased ($a$), PolicyBased (IPS), PolicyBased (DR)) exhibit a slight improvement in policy value for existing actions as the percentage of new actions increases. This is because, with fewer existing actions, it becomes easier to identify the optimal action within the reduced set. Nevertheless, PONA consistently achieves the same or slightly better overall policy value compared to PolicyBased (DR), while also selecting new actions for any percentage of new actions. In comparison, PolicyBased (DR) and other traditional methods never select new actions, even as the percentage of new actions increases. PolicyBased (PI) is mostly worse than PONA and LCPI in every metric due to its reliance on linearity, which is severely violated in our experimental environment.

\paragraph{\textbf{How does PONA perform when Condition~\ref{condition:local_linearity_condition} is violated?}}
We next evaluate the robustness of the proposed methods against violations of their key condition, namely, local linearity. To test this, we increase the value of $\gamma$ in the definition of the synthetic reward function in the x-axis of Figure~\ref{fig:gamma_per}. The figure demonstrates that as $\gamma$ increases, the violations of the local linearity condition become more pronounced, leading to a gradual decrease in the overall policy value of LCPI. In contrast, existing methods that do not rely on any assumptions about the q-function exhibit little change in policy value with varying values of $\gamma$, which is reasonable. Most interestingly, PONA, which also relies on local linearity, is more robust to the condition's violation than LCPI. This occurs because its data-driven tuning process in Eq.~\eqref{eq:tuning_kappa} adjusts its hyperparameter $\kappa$ to assign greater weight to the DR component as the violation becomes more severe. This adjustment adaptively maintains the policy value for existing actions, remaining a high overall policy value. Even in scenarios where the ability to learn new actions is at its weakest, PONA performs comparably to RegBased ($f$) for new actions, significantly outperforming PolicyBased (DR) and other baseline methods in selecting new actions. These results highlight PONA's ability to maintain a high policy value and effectively learn new actions, even when its key condition is severely violated

\paragraph{\textbf{How does PONA perform when varying the constraints about the proportions of new actions during parameter tuning?}}
In the previous experiments, we did not impose any constraints on the behavior of the learned policy when tuning the weight parameter $\kappa$ of PONA (i.e., always setting $\rho_L = -\infty$ and $\rho_U = \infty$ in Eq.~\eqref{eq:tuning_kappa}). As a result, the tuning procedure selected the $\kappa$ value that simply maximized the overall policy value in the validation data.

In Figure~\ref{fig:rho}, we report the overall policy value and the proportions of existing and new actions under the learned policies for varying lower limit $\rho_L$ in Eq.~\eqref{eq:tuning_kappa}, while keeping $\rho_U = \infty$.\footnote{The proportions of existing and new actions are respectively defined as $\mE_{p(x)} [\sum_{a\in\calA_{existing}} \pi_{\theta}(a|x)]$ and $\mE_{p(x)} [\sum_{a\in\calA_{new}} \pi_{\theta}(a|x)]$.} Note that we include the metrics for PONA ($\rho_L=-\infty$) and PONA ($\kappa=1.0$) as references in the figure. This experiment allows us to investigate how well PONA’s behavior can be controlled through the tuning of $\kappa$. The right plot in Figure~\ref{fig:rho} demonstrates that, as the lower limit on the proportion of new actions becomes larger, the proportion of new actions under \textbf{PONA (with varying $\rho_L$)} increases. This result shows that we can control the behavior and aggressiveness of the learned policy by appropriately formulating the tuning process of the weight parameter $\kappa$. It is also reasonable that the overall policy value of PONA decreases with larger $\rho_L$, as this increases the reliance of PONA on the local linearity condition.

\subsection{Real-World Data}
To evaluate the real-world applicability of PONA, we finally assess its performance on the KuaiRec dataset~\cite{gao2022kuairec},\footnote{https://kuairec.com/} which are derived from the recommendation logs of a video-sharing mobile app. This dataset includes a matrix with 1,411 users and 3,327 videos, where nearly 100\% of the rewards are observable. Consequently, we can directly use the observed rewards to conduct experiments on OPL, which makes this dataset the most suitable real-world public data for our experiments. We use user information as the context $x$ and video IDs as the actions $a$. We then rely on video tags and video categories to construct action features $f(a)$. More specifically, the video categories follow a three-level hierarchical structure. Combining the video tags and categories creates a four-dimensional discrete action feature space. To reduce complexity, we extract 15 factors for each action feature, which results in $|\calA|=117$. We use the watch ratio, defined as the play duration divided by the video duration, as the reward $r$. Following the synthetic experiment, we use the softmax function to define the logging policy with $\beta = 0.05$.

\paragraph{\textbf{Results.}} We compare the overall policy values learned by each OPL method across 50 simulations with different random seeds. All other settings match those of the synthetic experiments. Figures~\ref{fig:real_n_data} and~\ref{fig:real_n_actions} present the results with varying training data sizes and percentages of new actions. The results demonstrate trends similar to those observed in the synthetic experiments. PONA achieves the overall policy value comparable to PolicyBased (DR). PONA also successfully identifies promising new actions and performs much better than RegBased ($f$) in terms of \textbf{policy value per new action}. Note that typical methods including PolicyBased (DR) did not choose new actions at all on the real data as well. RegBased ($f$) and PolicyBased (PI) perform equally or worse compared to the uniform random policy in selecting new actions. This is because the real-world dataset contains a larger number of values for each action feature, making it more challenging to learn the rewards of new actions via mere regression or based on linearity.

\section{Conclusion}
In this paper, we tackled the challenge of off-policy learning (OPL) with new actions. Existing methods cannot choose new actions at all simply because we do not observe any data about them. We first proposed the Local Combination PseudoInverse (LCPI) estimator for the policy gradient, which relaxes the restrictive conditions of the PseudoInverse method by accounting for interactions among action features, enabling the more effective selection of new actions. To balance the strengths of LCPI for new actions and Doubly Robust for existing actions, we also introduced the Policy Optimization for Effective New Actions (PONA) algorithm. By integrating these estimators through a weight parameter, PONA efficiently selects new actions while maintaining strong overall performance.

\bibliographystyle{ACM-Reference-Format}
\balance
\bibliography{ref}

\appendix

\onecolumn
\allowdisplaybreaks
\raggedbottom

\section{Extended Discussion on Related Work} \label{app:related}
This section discuss the related work more comprehensively than Section~\ref{sec:related}.

\paragraph{\textbf{Off-Policy Evaluation and Learning.}}
Off-Policy Evaluation (OPE) and Learning (OPL)~\cite{dudik2014doubly, wang2017optimal, liu2018breaking, farajtabar2018more, su2019cab, su2020doubly, kallus2020optimal, metelli2021subgaussian, saito2022off, saito2023off} aim to evaluate and optimize decision-making policies using only historical logged data, without the need for direct testing in an online environment. Various methods and estimators have been developed in OPE to control the bias and variance of the estimation of the performance of new policies~\cite{su2020doubly, su2019cab, swaminathan2015counterfactual, swaminathan2015self, dudik2014doubly, farajtabar2018more, kallus2020optimal, kiyohara2022doubly, kiyohara2023off, metelli2021subgaussian, saito2021counterfactual, wang2017optimal, kiyohara2024off}. These advancements in OPE methodology have led to highly accurate estimation of the policy value in the standard setting without new actions~\cite{uehara2022review,saito2021evaluating,saito2021open}. In OPL, these estimation techniques enable accurate estimation of the policy gradient and enhance the resulting effectiveness of policy learning from only logged data~\cite{saito2024potec}. However, most existing methods in OPL rely crucially on the full support condition (Condition~\ref{condition:full_support}) and focus solely on selecting the most effective actions within those already present in the logged data~\cite{sachdeva2020off}. Therefore, they cannot deal with new actions that are potentially available in many decision-making problems and severely violate full support. In this paper, we propose a novel formulation and algorithm that leverage action features to address OPL with new actions. In particular, our algorithm can identify and choose effective new actions without sacrificing the overall policy performance, while existing methods cannot handle new actions even with the availability of the action features.

\paragraph{\textbf{OPE for Slate Bandits.}}
In this study, we address the challenge of incorporating new actions by considering OPL with action features. Specifically, we tackle realistic scenarios where each action is characterized by multi-dimensional action features, and a single reward is observed for each action. By mapping actions to slates and features of an action to slots of a slate, we can establish a connection between our problem with action features and OPE for slate bandits, initially explored by~\citet{swaminathan2017off}.

Slate bandits consider selecting a combinatorial action called \textit{slate} composed of multiple sub-actions. For example, in medical treatment scenarios, one may aim to select the optimal combination of drug doses to improve patient outcomes. While these systems often have access to extensive logged data, an accurate OPE remains challenging due to the exponential increase in variance associated with slate-wise importance weighting~\citep{swaminathan2017off, su2020doubly, kiyohara2024off}. To address this challenge, several versions of \textit{PseudoInverse} (PI) estimators have been proposed~\citep{su2020doubly, swaminathan2017off, vlassis2021control}. They use slot-wise importance weights to relax the support condition and reduce the variance of typical estimators such as IPS. They have also been proven to enable unbiased OPE under the \textit{linearity} assumption on the q-function. Linearity requires that the q-function be linearly decomposable, ignoring interaction effects among different slots. While PI often outperforms IPS under linear reward structures, it introduces significant bias when the linearity assumption does not hold~\cite{kiyohara2024off,shimizu2024effective}.

To effectively handle new actions in OPL by leveraging action features, we extend the idea of PI to estimate the policy gradient by relaxing the linearity assumption on the reward. This relaxation is crucial for dealing with the real-world non-linearities in reward functions during policy learning. It should be noted that the key contributions of our work are substantially different from those of existing studies for slate bandits~\cite{swaminathan2017off, kiyohara2024off, vlassis2021control, chaudhari2024distributional}. We focus on the problem of policy \textit{learning with new actions}. In contrast, the existing literature around slate OPE addresses only the problem of policy \textit{evaluation without considering new actions}.

\paragraph{\textbf{OPL with Deficient Actions.}}
Another setting similar to new actions is OPL with deficient actions~\cite{sachdeva2020off,felicioni2022off}, which refer to the actions that have zero probability of being selected under the logging policy for \textit{certain} contexts $x$. Rigorously, the set of deficient actions is written as $\{a \in \calA \mid \exists x \in \calX, \pi_0(a \mid x) = 0\}$. Deficient action is a broader concept that encompasses the new actions we consider as a special case. Specifically, a deficient action $a$ may still be observed for some contexts $x$ under the logging policy $\pi_0$, whereas new actions are never observed for any context $x$. This distinction becomes significant when estimating the q-function, $\hat{q}(x, a)$. The existence of deficient actions still allows for learning the q-function $q(x, a)$ because we observe every action somewhere in the logged data. However, with the presence of new actions, it becomes extremely challenging to estimate the q-function because we do not observe even a single data point for such actions. Due to this distinction, existing OPL approaches to deal with deficient actions~\cite{sachdeva2020off,felicioni2022off} cannot handle new actions. This motivates us to develop new methods specifically dealing with new actions without sacrificing the overall quality of policy learning.

\section{Omitted Proofs}
Here, we provide the derivations and proofs that are omitted in the main text.

\section{Proof of Unbiasedness of LCPI}

We provide a detailed proof that LCPI is unbiased under Conditions~\ref{condition:local_linearity_condition} and \ref{condition:local_combination_support}, mirroring the structure of Proposition~1 in \cite{swaminathan2017off} while incorporating our own notation. We begin with a key lemma and then use it to establish the final result. To prove this, let us define the following action sets:
\begin{align*}
\calA_{\pi_0}(x) &:= \{a \in \calA \mid \pi_0(a|x) > 0 \}, \\
\calA_{\text{LCS}}(x) &:= \{a \in \calA \mid \pi_0(f_{1:s}(a)|x) > 0, \pi_0(f_l(a)|x) > 0, \forall l \in [1, \ldots, d] \}.
\end{align*}

\begin{lemma}\label{lem:slatevalue}
If Condition~\ref{condition:local_linearity_condition} holds and $a \in \calA_{\text{LCS}}(x)$, then $q(x, a) = \mathbb{I}_a^{T} \Gamma_{\pi_0, x}^\dagger \theta_{\pi_0, x}$, where $\theta_{\pi_0, x} := \mathbb{E}_{\pi_0(a|x)}\left[\mathbb{I}_a q(x, a) \right]$.
\end{lemma}

\begin{proof}
By Condition~\ref{condition:local_linearity_condition}, there is a fixed vector $\phi_x$ such that for each action $a$,
\begin{align*}
q(x, a) = \mathbb{I}_a^{T} \phi_x.
\end{align*}

Define a binary matrix $M \in \{0,1\}^{|\calA_{\text{LCS}}(x)| \times (dm + m^s)}$, whose rows are the row vectors $\mathbb{I}_a^{T}$ for each $a \in \calA_{\text{LCS}}(x)$. The vector $M \phi_x$ enumerates $q(x, a)$ for all $a \in \calA_{\text{LCS}}(x)$. Let $\mathrm{Null}(M)$ denote the null space of $M$, and let $\Pi$ denote the orthogonal projection onto $\mathrm{Null}(M)$. Define $\phi_x^\star := (I - \Pi) \phi_x$. Since $M \phi_x = M \phi_x^\star$, it follows that
\begin{align*}
q(x, a) = \mathbb{I}_a^{T} \phi_x = \mathbb{I}_a^{T} \phi_x^\star \quad \text{for all } a \in \calA_{\text{LCS}}(x).
\end{align*}

From the definition in Section~\ref{sec:pona}, we have
\begin{align}\label{eq:vprx}
\theta_{\pi_0, x} = \mathbb{E}_{\pi_0(a|x)}\left[\mathbb{I}_a q(x, a)\right] = \mathbb{E}_{\pi_0(a|x)}\left[\mathbb{I}_a \mathbb{I}_a^{T} \phi_x\right] = \Gamma_{\pi_0, x} \phi_x.
\end{align}

Observe that
\begin{align*}
\mathrm{Null}(\Gamma_{\pi_0, x}) &= \left\{v \mid v^{T} \Gamma_{\pi_0, x} v = 0 \right\} \\
&= \left\{v \mid \mathbb{E}_{\pi_0(a|x)}\left[v^{T} \mathbb{I}_a \mathbb{I}_a^{T} v \right] = 0 \right\} \\
&= \left\{v \mid \mathbb{I}_a^{T} v = 0, \forall a \in \calA_{\pi_0}(x) \right\} \\
&= \left\{v \mid \mathbb{I}_a^{T} v = 0, \forall a \in \calA_{\text{LCS}}(x) \right\} \\
&= \mathrm{Null}(M),
\end{align*}
where the first step follows from the positive semi-definiteness of $\Gamma_{\pi_0, x}$. 
Since $\mathrm{Null}(\Gamma_{\pi_0, x}) = \mathrm{Null}(M)$, it follows from Eq.~\eqref{eq:vprx} that
\begin{align*}
\theta_{\pi_0, x} = \Gamma_{\pi_0, x} \phi_x^\star.
\end{align*}
Moreover, from the definition of $\phi_x^\star$, we know that $\phi_x^\star \perp \mathrm{Null}(\Gamma_{\pi_0, x})$.
Thus, by the definition of the pseudoinverse,
\begin{align*}
\Gamma_{\pi_0, x}^\dagger \theta_{\pi_0, x} = \phi_x^\star.
\end{align*}

Putting this together, for $a \in \calA_{\text{LCS}}(x)$:
\begin{align*}
q(x, a) = \mathbb{I}_a^{T} \phi_x^\star = \mathbb{I}_a^{T} \Gamma_{\pi_0, x}^\dagger \theta_{\pi_0, x}.
\end{align*}
This completes the proof of Lemma~\ref{lem:slatevalue}.
\end{proof}

We now use the above Lemma to prove the unbiasedness of LCPI.
\begin{proof}
\begin{align*}
\mathbb{E}_{\calD} \left[\nabla_\theta \widehat{V}_{\mathrm{LCPI}}(\pi_\theta; \mathcal{D})\right] &= \mathbb{E}_{\calD}\left[\frac{1}{n}\sum_{i=1}^n \left(\sum_{a \in \mathcal{A}}\pi_\theta(a|x_i) \nabla_\theta \log\pi_\theta(a|x_i) \mathbb{I}_a^{T} \right)\Gamma_{\pi_0, x_i}^\dagger \mathbb{I}_{a_i} r_i\right] \\
&= \mathbb{E}_{p(x)\pi_0(a'|x)p(r|x, a')}\left[\left(\sum_{a \in \mathcal{A}}\pi_\theta(a|x) \nabla_\theta \log\pi_\theta(a|x) \mathbb{I}_a^{T} \right)\Gamma_{\pi_0, x}^\dagger \mathbb{I}_{a'} r\right] \\
&= \mathbb{E}_{p(x)}\left[\left(\sum_{a \in \mathcal{A}}\pi_\theta(a|x) \nabla_\theta \log\pi_\theta(a|x) \mathbb{I}_a^{T} \right)\Gamma_{\pi_0, x}^\dagger \mathbb{E}_{\pi_0(a'|x)}\left[\mathbb{I}_{a'} q(x, a')\right]\right] \\
&= \mathbb{E}_{p(x)}\left[\left(\sum_{a \in \mathcal{A}}\pi_\theta(a|x) \nabla_\theta \log\pi_\theta(a|x) \mathbb{I}_a^{T} \right)\Gamma_{\pi_0, x}^\dagger \theta_{\pi_0, x}\right] \\
&= \mathbb{E}_{p(x)}\left[\sum_{a \in \mathcal{A}}\pi_\theta(a|x) \nabla_\theta \log\pi_\theta(a|x) q(x, a)\right] \quad \because \text{Condition~\ref{condition:local_combination_support}, Lemma~\ref{lem:slatevalue}} \\
&= \nabla_\theta V(\pi_\theta).
\end{align*}
\end{proof}

\section{Experiment Details and Additional Result}
This section describes the detailed experiment settings and reports additional results. 

\subsection{Synthetic Experiments}
\paragraph{\textbf{Detailed Setup.}} We describe synthetic experiment settings in detail. First, we define the space of existing actions $\mathcal{A}_{existing}$ as follows:

\[
\mathcal{A}_{existing}:= \mathcal{A}_{\text{base}}^{\text{IS}} \cup \mathcal{A}_{\text{base}}^{\text{LCS}}(s) \cup \mathcal{A}_{\text{random}}
\]
where $\mathcal{A}_{\text{base}}^{\text{IS}}$ and $\mathcal{A}_{\text{base}}^{\text{LCS}}(s)$ are artificially constructed sets of actions that satisfy the Independent Support condition (\ref{condition:independent_support}) and Local Combination Support condition (\ref{condition:local_combination_support}), respectively, as below:
\begin{align*}
    &\mathcal{A}_{\text{base}}^{\text{IS}} := \{a=(f_{1, i},f_{2, i},\cdots, f_{d, i}) | \forall i \in m\}, \\ &\mathcal{A}_{\text{base}}^{\text{LCS}}(s) := \left\{a=(f_{1:s}, f_{s+1, 1}\cdots, f_{d, 1}) | \forall f_{1:s} \in \prod_{l=1}^s \mathcal{F}_l \right\}.
\end{align*}
Note that $f_{j, i}$ represent the $i$-th feature value in the $j$-th dimension, where $j \in [d]$ and $i \in [m]$. This construction of the action space ensures that the logging policy selects each feature at least once, satisfying the Independent Support condition. It also ensures that the logging policy selects all combinations of the first $s$ dimensions, satisfying the Local Combination Support condition. Additionally, $\mathcal{A}_{\text{random}}$ contains a set of actions randomly sampled from $\mathcal{A}$.

\subsection{Real-World Experiments}
\paragraph{\textbf{Detailed Setup.}} 
We describe the real-world experiment settings on KuaiRec \cite{gao2022kuairec} in detail. We use user information as the context $x$ and reduce the feature dimensions using PCA implemented in scikit-learn \cite{pedregosa2011scikit}. We define the reward as the user watch ratio, clipped at the 99th percentile and normalized. 

There are four types of action features: video tags, first-level categories, second-level categories, and third-level categories. For the combination features of LCPI and PONA (ranging from 1 to $s$ dimensions), we utilize three types: video tags, first-level categories, and third-level categories. For existing actions, we prioritize those with overlapping features in the 1- to $s$-dimensional combination features to ensure that Local Combination Support is satisfied as much as possible.

\allowdisplaybreaks

\end{document}